\documentclass[twoside, 11pt]{article}

\usepackage{jmlr2e}
\usepackage[parfill]{parskip}
\usepackage[]{mdframed}
\usepackage{booktabs}
\usepackage{multirow}
\usepackage{pifont}
\usepackage{xcolor}
\usepackage{makecell}

\usepackage{url, hyperref}
\usepackage{algorithm, algpseudocode} %

\usepackage{amsthm, amsmath, amssymb} %amsfonts, amsmath, bm
\usepackage[shortlabels]{enumitem}
\usepackage{comment}
\usepackage{color}
\usepackage{multirow}
\usepackage{xspace}
\usepackage{rotating}
\usepackage{graphicx}
\usepackage{caption}
\usepackage{xcolor}
\usepackage[normalem]{ulem}
\usepackage{tikz}

\usepackage{nicefrac}
\usepackage{mathtools}
\mathtoolsset{showonlyrefs,showmanualtags}

\theoremstyle{plain}
\newtheorem{theorem}{\textbf{Theorem}}

\newtheorem{corollary}[theorem]{Corollary}
\newtheorem{proposition}[theorem]{Proposition}

\newtheorem{example}{Example}

\theoremstyle{definition}

\newtheorem{assumption}{Assumption}
\newtheorem{fact}{Fact}

\renewcommand{\cite}{\citep}

\newcommand{\PP}{\mathbb{P}}
\newcommand{\EE}{\mathbb{E}}

\newcommand{\II}{\mathbf{I}}

\newcommand{\Indi}{\mathbb{I}}

\DeclareMathOperator*{\argmin}{arg\,min}

\newcommand{\Fcal}{\mathcal{F}}

\newcommand{\Hcal}{\mathcal{H}}

\newcommand{\Scal}{\mathcal{S}}

\newcommand{\Acal}{\mathcal{A}}

\newcommand{\Bcal}{\mathcal{B}}
\newcommand{\Ecal}{\mathcal{E}}
\newcommand{\Ocal}{\mathcal{O}}

\newcommand{\Rmax}{R_{\max}}

\newcommand{\Vmax}{V_{\max}}
\newcommand{\RR}{\mathbb{R}}
\newcommand{\Xcal}{\mathcal{X}}
\newcommand{\Ycal}{\mathcal{Y}}

\newcommand{\bs}{\textbf{b}}

\newcommand{\PO}{\Gamma}
\newcommand{\RO}{\textsc{1-Reset}}
\newcommand{\ROR}{\textsc{Single-Reset}\xspace}
\newcommand{\RAR}{\textsc{Repeated-Reset}\xspace}

\newcommand{\oc}{\rho}
\renewcommand{\em}{E}
\newcommand{\bstrue}{\bs^\star}
\newcommand{\ii}[1]{{(#1)}}
\newcommand{\ts}{\tilde{s}}
\newcommand{\TV}{D_{\textup{TV}}}
\newcommand{\lss}{\textsc{latent state-based selection}\xspace}
\newcommand{\LSS}{\textsc{Latent State-based Selection}\xspace}

\newcommand{\os}{\textsc{observation-based selection}\xspace}
\newcommand{\OS}{\textsc{Observation-based Selection}\xspace}
\newcommand{\Os}{\textsc{Observation-based selection}\xspace}
\newcommand{\POreal}{\PO_\star}

\definecolor{DarkGreen}{rgb}{0, 0.6, 0} 
\newcommand{\greencheck}{{\color{DarkGreen}\checkmark}}
\newcommand{\redcross}{{\color{red}\ding{56}}}

\begin{document}
	
\title{Selecting Belief-State Approximations \\
in Simulators with Latent States}

\author{%
\name Nan Jiang \email nanjiang@illinois.edu \\
\addr University of Illinois Urbana-Champaign
}

\maketitle

\newcommand{\nan}[1]{{\color{red} [NJ: #1]}}

\newcommand{\para}[1]{\paragraph{#1}}%{\vspace*{1pt} \noindent \textbf{#1}~}

\renewcommand*{\theHsection}{\thesection}
\renewcommand*{\theHsubsection}{\thesubsection}

\begin{abstract}
State resetting is a fundamental but often overlooked capability of simulators. It supports sample-based planning by allowing resets to previously encountered simulation states, and enables calibration of simulators using real data by resetting to states observed in real-system traces. While often taken for granted, state resetting in complex simulators can be nontrivial: when the simulator comes with latent variables (states), state resetting requires sampling from the posterior over the latent state given the observable history, a.k.a.~the belief state \citep{silver2010monte}. While exact sampling is often infeasible, many approximate belief-state samplers can be constructed, raising the question of how to select among them using only sampling access to the simulator. 

In this paper, we show that this problem reduces to a general conditional distribution-selection task and develop a new algorithm and analysis under sampling-only access. Building on this reduction, the belief-state selection problem admits two different formulations: \lss, which directly targets the conditional distribution of the latent state, and \os, which targets the induced distribution over the observation. Interestingly, these formulations differ in how their guarantees interact with the downstream roll-out methods: perhaps surprisingly, \os may fail under the most natural roll-out method (which we call \ROR) but enjoys guarantees under the less conventional alternative (which we call \RAR). Together with discussion on issues such as distribution shift and the choice of sampling policies, our paper reveals a rich landscape of algorithmic choices, theoretical nuances, and open questions, in this seemingly simple problem. 
\end{abstract}

\section{Introduction}
Applying reinforcement learning (RL) to real-world domains often relies on training and evaluating policies in simulation. A basic functionality of simulation is \textit{state resetting/loading}, i.e., setting the simulator into a state that is either previously encountered in simulation or observed in the real system. The former enables sample-based planning---for example, MCTS methods roll-out multiple trajectories from the same state \citep{kocsis2006bandit, browne2012survey}---while the latter allows one to calibrate the simulator by comparing its predicted next-state to what occurs in reality \citep{liu2025model}.

Despite often taken for granted in research papers, state resetting can be highly nontrivial in complex simulators, especially when they come with \textit{latent variables} that are introduced to model the generative processes of the observables but cannot be measured in the real systems. Na\"ive approaches, such as loading the saved latent states (e.g., loading previously dumped RAM state \citep{ecoffet2019go}), is not only infeasible in real systems---since the values of the latent variables are nowhere to be found---but also problematic for resetting to a previous  simulation state; for example, policies trained with such na\"ive resetting may depend their actions on privileged information in the latent states, and thus may face performance degradation when distilled to a policy that operators only on observable information \citep{jiang2019value, weihs2021bridging}. 
The correct formulation is to view the simulator as a POMDP, which induces an MDP where the observable \textit{history} (i.e., the sequence of observations and actions) is treated as the state. State resetting amounts to using the observable history to set the values of the latent variables. Mathematically, we should \textbf{sample from the posterior distribution of latent variables conditioned on the observable history}, a.k.a.~the \textit{belief state} of the POMDP \citep{silver2010monte}.

While the belief state is conceptually well-defined for any POMDP, exact sampling from belief states can be computationally demanding, especially when the observation space and the latent state space are high-dimensional and the latent dynamics and the emission process are complex black-boxes. To address this challenge, algorithms for approximately sampling from such a distribution have been proposed: for example, the problem can be viewed as an instance of approximate Bayesian computation (ABC), and \citet{silver2010monte} apply rejection sampling to approximate the belief state. However, rejection sampling, when implemented exactly, incurs exponential-in-horizon sample complexity even when the observation space is finite and small, and requires problem-specific heuristics to trade off accuracy for efficiency. Likewise, techniques from related areas such as Simulation-based Inference (SBI) also come with design choices that need to carefully tuned. This naturally gives rise to the following question:
\begin{center}
\textit{Given multiple approximations to the belief state, how can we select from them, and what theoretical guarantees can be obtained?}
\end{center}
In this paper, we explore the multi-faceted nature of this problem. We first show that finding a good belief-state approximation can be reduced to a general conditional distribution-selection problem, and provide a new algorithm and an analysis for the latter under only sampling access to the candidate conditionals (Section~\ref{sec:conditional}). Building on this reduction, we then show that belief-state selection itself admits two distinct formulations: \textbf{\lss}, which directly targets the posterior of latent state given history, and \textbf{\os}, which targets the induced observable transition model (Section~\ref{sec:os}). % 
These two formulations behave differently in the presence of redundant latent variables and, perhaps more importantly, interact in subtle but consequential ways with how we use the selected belief state in downstream tasks. Perhaps surprisingly, we show that, when the selected belief-state approximation is used to estimate Q-values via Monte-Carlo roll-outs, \os can have degenerate behavior under the most natural roll-out procedure which we call \textbf{\ROR}, but enjoy guarantees under the counter-intuitive \textbf{\RAR} roll-out (Section~\ref{sec:rollout}; see also Table~\ref{tab:comparison}). We conclude the paper with further discussions on the issues related to distribution shift and the design of sampling policies. Collectively, these results and insights reveal a rich landscape of choices and nuances in this seemingly simple problem. 

\begin{comment}

The goal is to select from $\{P_i(Y|X)\}_{i=1}^m$ a conditional distribution that approximates the true posterior $P(Y|X)$. This formulation is generic and also captures problems such as model selection over conditional generations (e.g, prompt-based image generation), and we will review several methods and their theoretical guarantees in Section \nan{sec}. 

The story, however, does not end here. The sequential decision-making aspect adds interesting twists to the problem, revealing alternative methods to the aforementioned mechanical application of the ABC/SBI framework, as well as additional concerns in sequential simulation and out-of-distribution generalization when applying the belief-state approximations. 

\paragraph{Alternative Methods.} In Section \nan{sec}, we show that instead of modeling belief states, we can let $Y$ be the next observation, and candidate posteriors for $Y$ given history $X$ can be induced from combining belief-state approximations with one-step simulation. Section~\nan{sec} discusses methods more specific to certain tasks such as policy evaluation. \nan{multi-step simulation?}

\paragraph{Additional Concerns.} Unlike a standard ABC/SBI problem where the marginal of $X$ is often given, here we need to select a \textit{simulation} policy to generate histories $X$, and the result of selection (a belief-state approximation) is often applied on a different distribution of $X$. The selection of simulation policy and the analyses of out-of-distribution guarantees pose several open questions. 

\end{comment}

\section{Preliminaries}
\subsection{MDPs and POMDPs} \label{sec:mdp}
\paragraph{Markov Decision Processes (MDPs)} We consider $H$-step finite-horizon MDPs, defined by the state space $\Scal$, action space $\Acal$, reward function $R:\Scal \to [0, \Rmax]$, transition function $P: \Scal\times\Acal\to\Delta(\Scal)$,  initial state distribution $\oc_0 \in \Delta(\Scal)$; here we assume $\Scal, \Acal$ are finite for convenience, and $\Delta(\cdot)$ is the probability simplex. 
We adopt the convention of layered state space that allows for time-homogeneous notation for time-inhomogeneous quantities: that is, let $\Scal = \bigcup_{0\le t \le H} \Scal_t$, where $\oc_0$ is supported on $\Scal_0$. $P(s'|s,a)$ is always $0$ unless $s\in \Scal_t$ and $s'\in\Scal_{t+1}$, thus any state that may appear as $s_t$ always belongs to $\Scal_t$. Any policy $\pi: \Scal\to\Delta(\Acal)$ (note that this captures time-inhomogeneous policies) induces a distribution over the trajectory (or episode) $s_0, a_0, r_0, \ldots, s_{H-1}, a_{H-1}, r_{H-1}, s_H$ by the following generative process: $s_0 \sim \oc_0$, $\forall t \ge 0$, $a_t \sim \pi(\cdot | s_t)$, $s_{t+1} \sim P(\cdot|s_t, a_t)$, $r_t = R(s_{t+1})$. We use $\PP_{P^\pi}[\cdot]$ and $\EE_{P^\pi}[\cdot]$ to denote this distribution and the expectation w.r.t.~it.\footnote{The distribution also depends on $\oc_0$ and the reward function $R$, but the different models we will consider often only differ in the transition, so we use the subscript to emphasize the dependence on transition.} 

A standard objective that measures the performance of a policy $\pi$ is the expected return, $J(\pi) := \EE_{P^\pi}[\sum_{t\ge0} r_t]$.  Let $\Vmax = H\Rmax$ denote the range of the cumulative rewards. %Two common tasks in RL is policy evaluation and policy optimization, i.e., estimating $J(\pi)$ for a given $\pi$, and finding $\pi$ that optimizes $J(\pi)$, respectively. 
As a central concept in RL, a (Q)-value function is defined as $Q_{P}^\pi(s,a) = \EE_{P^\pi}[\sum_{t'\ge t} r_t | s_t = s, a_t = a]$ for $s \in \Scal_t$. 

\paragraph{Partially Observable MDPs (POMDPs)} A POMDP $\Gamma$ is specified by an underlying MDP plus an emission process, $\em: \Scal\to\Delta(\Ocal)$, which generates observation $o_t \in \Delta(\Ocal)$ based on a latent state $s_t \in\Scal$ as $o_t \sim \em(\cdot|s_t)$; similar to before we assume $\Ocal$ is layered, i.e., $o_t$ is always supported on $\Ocal_t$. An episode in a POMDP is generated similarly to the MDP: $s_0 \sim \oc_0$, and at any time step $t$, an observation is generated as $o_t \sim \em(\cdot|s_t)$, the agent takes action $a_t$ that only depends on the observable history $o_{0:t} := \{o_0, \ldots, o_t\}$ and $a_{0:t-1}$. Then, a latent transition $s_{t+1} \sim P(\cdot|s_t, a_t)$ occurs and a reward $r_t$ is generated, and so on and so forth. We assume that the information of reward $r_t$ is always encoded in $o_{t+1}$, so with a slight abuse of notation we write $r_t = R(s_{t+1}) = R(o_{t+1})$. POMDPs are often used to model processes where the observations violate the Markov property. That is, we only observe $o_t$ in the real system, and introduce $s_t$ to explain the dynamics and evolution of $o_t$. In this case, $s_t$ are latent states that are not observed in the real-system traces. 

\paragraph{Belief States and History-based MDP}
A key concept in POMDPs is the \textit{belief state}, $\bstrue(s| \tau): = \PP_{\PO}[s_t = s \mid \tau_t = \tau]$, where $\tau_t = (o_{0:t}, a_{0:t-1})$ is an \textit{observable history}. %(Here we make the simplification that rewards reveal no additional information about latent states, otherwise they need to be included in the history.)\footnote{One way to guarantee this is to augment $o_t$ to include the information of $r_{t-1}$.} 
% We write $\PP_{\PO}$ without a policy subscript because the quantity is policy-independent as long as actions only depend on the observables.\footnote{\nan{confounding}} %It is well established that there always exists an optimal policy for the POMDP that only depends on the belief state. \nan{cite}
% 
It is useful to think of the evolution of the observable variables of a POMDP as a \textit{history-based MDP}. That is, let the $t$-step history $\tau_t$ be the state, and upon action $a_t$, the reward $r_t$ and next-state $\tau_{t+1}$ are generated as 
$$
s_t \sim \bstrue(\cdot|\tau_t), s_{t+1} \sim P(\cdot|s_t, a_t), r_t = R(s_{t+1}), o_{t+1} \sim \em(\cdot|s_{t+1}), ~~ \tau_{t+1} = (o_{0:t+1}, a_{0:t}).
$$
We use $M_{\PO}(o_{t+1}|\tau_t, a_t)$ to denote the conditional distribution and the induced MDP dynamics. (Note that since reward $r_t$ is encoded in $o_{t+1}$, it can also be determined from $\tau_{t+1}$ and thus is consistent with our MDP formulation.) This MDP naturally fits the layered convention, where $\Hcal_t$, the space of $\tau_t$, is the $t$-th step state space. This way, any history-dependent policy is simply a Markov policy in the history-based MDP, $\pi: \Hcal\to \Delta(\Acal)$. Concepts such as value functions for a POMDP can be immediately defined through its history-based MDP, that is, when we mention the Q-function in a POMDP such as $Q_{\PO}^\pi$, what we mean is
$
Q_{\PO}^\pi = Q_{M_{\PO}}^\pi. 
$

\paragraph{Additional Notation} For two distributions $p,q \in \Delta(\Xcal)$, define their Total-Variation (TV) distance as $\TV(p, q):= \sum_{x\in \Xcal} |p(x) - q(x)|/2$, and let $\|p/q\|_\infty:= \max_{x\in\Xcal} p(x)/q(x)$.

\subsection{Model Selection of Belief-State Approximations}
As mentioned in the introduction, we are interested in complex simulators where, when modeled as a POMDP, the latent-state and the observation spaces $\Scal$ and $\Ocal$ are potentially very large, and the latent transition and the emission process $P$ and $\em$ are complex black-boxes, to which we only have sampling access. While the notion of belief state, $\bstrue$, is conceptually and information-theoretically well-defined, it is not easy to access them in a computationally efficient manner, and methods from ABC, SBI, and particle filtering may be used to approximate the said belief state \citep{cranmer2020frontier}. Since  these methods often require domain-specific design choices and heuristics, the model-selection problem naturally arises: given a candidate set of belief-state approximations $\mathcal{B} = \{\bs^\ii{i}\}_{i=1}^m$ with $\bs^\ii{i} : \Hcal \to \Delta(\Scal)$, we are interested in selecting the best approximation by interacting with the simulator. Throughout the paper, we will assume \textit{realizability} as a simplification: %\vspace*{0.5em}

\begin{assumption}[Realizability]
$\bstrue \in \Bcal$.
\end{assumption}

\section{Selection of $s_t | \tau_t$ (``\LSS'')}
\label{sec:conditional}

The problem of selecting/learning the posterior distribution in a computationally-efficient manner is closely related to Simulation-based Inference (SBI). As a standard approach in SBI,  we can generate trajectories with latent states in the form of $(s_{0:H}, o_{0:H}, a_{0:H-1})$ using some \textit{behavior policy} $\pi_b:\Hcal \to \Delta(\Acal)$, and obtain $\tau_t = (o_{0:t}, a_{0:t-1})$ and $s_t$ pairs for any $0\le t\le H$. The joint distribution of $(\tau_t, s_t)$ generated in this way satisfies that
$$
s_t \sim \bstrue(\cdot|\tau_t),
$$
and we write $\tau_t \sim \PO^{\pi_b}$ to denote that $\tau_t$ is generated from a trajectory induced by policy $\pi_b$ in POMDP $\PO$. 
On the other hand, for any candidate $\bs \in \Bcal$, we can also generate 
$$
\tilde{s}_t \sim \bs(\cdot|\tau_t)
$$
for each $\tau_t$ in the above dataset. Then, testing whether $\bs = \bstrue$ can be reduced to the problem of \textit{conditional 2-sample test} based on the samples $\{(\tau_t, s_t, \ts_t)\}$, that is, to tell whether $s_t|\tau_t$ and $\ts_t|\tau_t$ are identically distributed.\footnote{Strictly speaking, 2-sample test is different from and arguably harder than the selection problem, since we can leverage the realizability assumption in selection. } We call this approach \textit{\lss} to distinguish it from alternative approaches we will consider later.

\paragraph{Reduction to Joint 2-Sample Tests} A na\"ive approach is to reduce the \textit{conditional} test to a \textit{joint} test: we can simply test if $(\tau_t, s_t)$ is identically distributed as $(\tau_t, \ts_t)$. If the joints are identical, it implies that the conditionals are also identical on the supported $\tau_t$. Unfortunately, this approach comes with significant practical hurdles: 2-sample tests often involve some kind of discriminator class $\Fcal$ that need to be carefully designed \citep{gretton2012kernel}, which in this case operates on $\Hcal \times \Scal$. However, given that a history $\tau\in \Hcal$ is a combinatorial object of variable length, designing effective discriminators can be practically challenging. This begs the question:
\begin{center}
\textit{Can we design algorithms that do not rely on discriminators over the $\Hcal$ space?}
\end{center}

\subsection{Selection of $Y|X$ Conditionals with $Y$-only Discriminators}

We now provide a solution to the general problem of selecting from conditional distributions in the form of $P(Y|X)$, in a way that only requires discriminator classes operating on $Y$, avoiding the demanding task of feature engineering or neural architecture design over $X$ which are often complex combinatorial objects (such as histories) in our settings of interest. 

\paragraph{General Problem Formulation} We are given $n$ i.i.d.~$(X,Y)$ pairs, $\{(X_j, Y_j)\}$, sampled from a real joint distribution $(X,Y) \sim P^\star$, and the task is to select from $m$ candidate conditionals $P_i(y|x)$ where $P_{i^\star}(y|x) = P^\star(y|x)$ for some $i^\star \in [m]$. Computationally, we assume we can efficiently sample from $P_i(\cdot|x)$ for any given $x$, but we can only sample joints from $P^\star$. 

Inspired by Scheff\'e tournament \citep{devroye2001combinatorial}, we first consider the case of $m=2$ and later extend to general $m$ via a tournament procedure of pairwise comparison.

\paragraph{Pairwise Comparison between 2 Candidates} When $m = 2$, we propose the following procedure, which requires a discriminator class $\Fcal:\Ycal\to \{0, 1\}$ to serve as \textit{classifiers}: for each $X_j$ in the data sampled from $P^\star$, 

\begin{enumerate}
\item Sample $N$ i.i.d.~data points from $P_k(\cdot|X_j)$ for $k=1, 2$.
\item Use the above $2N$ data points to train a classifier $\hat f_j \in \Fcal$ that predicts whether $Y$ is sampled from $P_1(\cdot|X_j)$ or $P_2(\cdot|X_j)$. For theoretical analyses and presentation ease, we assume ERM on 0/1 loss, and adopt the convention (which is w.l.o.g.) that $P_{i^\star}$ gets label $1$.
\item Use $\hat f_j$ to classify the real $Y_j$. Additionally sample 1 data point from each of $P_1(\cdot|X_j)$ and $P_2(\cdot|X_j)$, denoted as $Y_j^{(1)}$ and $Y_j^{(2)}$, and classify them with $f_j$ as well.  
\end{enumerate}
Finally, we choose between $P_1, P_2$ based on
$
\argmin_{k \in \{1, 2\}} \left|\frac{1}{n} \sum_{j} \hat f_j(Y_j) - \frac{1}{n} \sum_j \hat f_j(Y_j^{(k)}) \right|.
$ 
For $m>2$, we perform the above procedure for each pair of candidate conditionals (data sampled from $P_k$ may be reused in multiple comparisons), and let $\hat f_j^{i, k}$ be the classifier trained to distinguish between $P_i(\cdot|X_j)$ and $P_k(\cdot|X_j)$. The final choice is
\begin{align} \label{eq:scheffe}
\argmin_{i \in [m]} \max_{k \in [m], k\ne i} \left|\frac{1}{n} \sum_{j} \hat f_j^{i, k}(Y_j) - \frac{1}{n} \sum_j \hat f_j^{i,k}(Y_j^{(i)}) \right|.
\end{align}

%Since $f_j$ predicts the probability for $P_{i^\star}$, our choice will be correct if $\frac{1}{n} \sum_{j} \hat f_j(Y_j) > 1/2$ (in which case we say $i^\star$ ``wins'' in the comparison, otherwise $i\ne i^\star$ wins). 
In words, for each real data point $X_j$, we draw ``synthetic data'' from the candidate conditionals  $P_i(\cdot|X_j)$ and $P_k(\cdot|X_j)$ to train a classifier, and apply it to a single ``real'' data point $Y_j$; in total, $n m(m-1) / 2 $ classifiers will be trained. When $i^\star \in \{i, k\}$,  the classifier learns to separate $Y$ generated using $P^\star = P_{i^\star}$ from that using the wrong conditional. Therefore, we may choose between $P_i$ and $P_k$ based on $\hat f_j^{i,k}(Y_j)$, which predicts whether $Y_j$ (sampled from $P^\star = P_{i^\star}$) is more likely to be produced by $P_i$ or $P_k$. %Intuitively, a nontrivial classifier is more likely to classify real samples $Y_j$ as having the same label as $P_{i^\star}$, so a natural idea is to choose between $i$ and $k$ based on , which predicts whether $Y_j$ is more likely : if $\hat f_j^{i,k}$. 
Of course, the signal from classifying a single data point $Y_j$ is weak and contains randomness, and aggregating such signals across all $j$ can reduce the noise and amplify the signal.

While the above idea is reasonable, it may run into issues when the data from $P_i$ and $P_k$ are not cleanly separable by $\Fcal$: the algorithm minimizes an overall error rate with a mixture of correct data (from $P_{i^\star}$) and incorrect data, but the classifier is eventually only applied to $Y_j$ from $P_{i^\star}$, meaning that the ultimate loss we suffer is the False Negative Rate of $\hat f_j^{i,k}$, which is only loosely controlled by the overall error rate. In contrast, we follow the spirit of Scheff\'e estimators and treat the classifier $\hat f_j^{i,k}$ as an approximate witness of the total-variation distance between $P_i(\cdot|X_j)$ and $P_k(\cdot|X_j)$, and make the final selection via the IPM loss in Eq.\eqref{eq:scheffe}, which leads to the relatively weak assumption in the theoretical analysis below.

\paragraph{Theoretical Guarantee} We now provide the assumptions and the theoretical guarantees for this procedure. The key assumption we need is that $\Fcal$ contains nontrivial classifiers that separate $P_i(\cdot|X_j)$ from $P_k(\cdot|X_j)$ when $i^\star \in \{i,k\}$, as formalized by the following assumption. 

\begin{assumption}[Expressivity of $\Fcal$] \label{asm:nontrivial_F}
Assume $\Fcal$ is a finite class. Define
$$
\textup{acc}_X^{i, i^\star}(f):= 1/2\cdot \Pr_{Y\sim P^\star(\cdot|X)}[f_X^{i, i^\star}(Y) = 1] + 1/2\cdot \Pr_{Y \sim P_i(\cdot|X)}[f_X^{i, i^\star}(Y) = 0].
$$
For any $i\ne i^\star$, let 
\begin{align} \label{eq:bayes-opt}
f_{x, \star}^{i, i^\star}(y) = \Indi[P_{i^\star}(y|x) > P_i(y|x)]    
\end{align}
be the Bayes-optimal classifier for distinguishing between $P_{i^\star}(\cdot|x)$ and $P_i(\cdot|x)$, and 
$$
\Ecal(i, i^\star) = \EE_{X \sim P^\star} \left[\textup{acc}_X^{i, i^\star}(f_{X, \star}^{i, i^\star}) \right] - 1/2.
$$
We assume the existence of $f_X^{i, i^\star} \in \Fcal$ that satisfies the following: for some $0 < \alpha \le 1$ that applies to all $i$, 
$$
\EE_{X \sim P^\star} \left[\textup{acc}_X^{i, i^\star}(f_X^{i, i^\star}) \right] \ge 1/2 + \alpha \, \Ecal(i, i^\star).
$$
\end{assumption}
$f_{X, \star}^{i, i^\star}$ (Eq.\eqref{eq:bayes-opt}) is the best possible classifier for separating $P_{i^\star}(\cdot|X) = P^\star(\cdot|X)$ from $P_i(\cdot|X)$ when $i\ne i^\star$, and we can get straightforward guarantees if we simply assume $f_{X, \star}^{i, i^\star} \in \Fcal$. Instead, we only assume $\Fcal$ realizes ``better-than-trivial'' classifier $f_{X}^{i, i^\star}$, and the rest of this assumption introduces definitions to quantify what ``better-than-trivial'' means. 

The $\textup{acc}_X^{i, i^\star}$ term is the classification accuracy, with the convention (which is w.l.o.g.) that $P_{i^\star}$ is assigned label $1$ and $P_i$ is assigned label $0$. Note that when $\Fcal$ is closed under negation ($1- f \in \Fcal, \forall f\in \Fcal$), it is trivial to find a classifier in $\Fcal$ with $1/2$ accuracy, so $\Ecal(i, i^\star)$ is a measure of how separable $P_i$ and $P_{i^\star}$ are; in fact, $\Ecal(i,i^\star) = \nicefrac{1}{2} \cdot \EE_{X \sim P^\star}[\TV(P_i(\cdot|X), P_{i^\star}(\cdot|X))]$. Given all these definitions, we require $\Fcal$ to contain a classifier $f_{X}^{i, i^\star}$ whose margin is only a multiplicative fraction of $\Ecal(i,i^\star)$, and this does not need to hold for every $X$, but only on average w.r.t.~the marginal of $X$. 

We are now ready to give the guarantee; the proof is deferred to Appendix~\ref{app:scheffe}.  

\begin{theorem}[Sample complexity]  \label{thm:scheffe-TV}
Under Assumption~\ref{asm:nontrivial_F}, for $\hat i$ identified by Eq.\eqref{eq:scheffe}, with probability at least $1-\delta$, 
$$
\EE_{X \sim P^\star}[
\TV(P_{\hat i}(\cdot|X), P^\star(\cdot|X))] \le \epsilon,
$$
as long as
$$
n = O\left(\frac{\log(m/\delta)}{\alpha^2\epsilon^2}\right), \qquad N = O\left(\frac{\log(mn|\Fcal|/\delta)}{\alpha^2\epsilon^2}\right).
$$
Invoked on $X = \tau_t$, $Y= s_t$, and $P^\star$ is distribution under behavior policy $\pi_b$, we have
\begin{align} \label{eq:belief-guarantee}
\EE_{\tau_t \sim \PO^{\pi_b}}[\TV(\bs(\cdot|\tau_t), \bstrue(\cdot|\tau_t))] \le \epsilon,
\end{align}
where $\tau_t \sim \PO^{\pi_b}$ is a partial trajectory naturally simulated in $\PO$ under policy $\pi_b$ without using resetting. 
\end{theorem}
 
\paragraph{Related Works} The above procedure is closely related to and a variant of the method of \citet{li2022minimax}, both of which can be viewed as the extension of Scheff\'e estimators to conditional distributions. The difference is that \citeauthor{li2022minimax} assume access to the density functions of $P_i(y|x)$, which allows them to explicitly compute the Bayes-optimal classifier in Eq.\eqref{eq:bayes-opt}. In contrast, we only allow blackbox sampling access to $P_i(\cdot|x)$, and approximate the above classifier using a discriminator class $\Fcal$ via training on sampled synthetic data. What we show is that $\Fcal$ does not need to realize the above Bayes-optimal classifier for every single $x$, and it suffices to have marginally-better-than-trivial classifiers in an average sense. Theoretically, \citeauthor{li2022minimax} show that the sample complexity of the conditional selection problem should not depend on the complexity of the $\Xcal$ space (see also \citet{bilodeau2023minimax}), which is echoed by our motivation of not having to design discriminators over $\Xcal$. In model-based RL, similar insights and Scheff\'e-style constructions have been used in learning model dynamics from IPM losses \citep{sun2019model}. The idea of discriminators to help learn or test conditional distributions are also found in the SBI literature \citep{lueckmann2021benchmarking}.

Outside belief-state approximation, our procedure may also be relevant to model selection in conditional generative models, such as prompt-based image generation. One potentially relevant property is that our procedure has a low sample-complexity burden on $n$, the number of ``real'' data points from $P^\star$. In belief-state approximation, both $n$ and $N$ can be increased by spending more computation; in tasks of learning from real datasets, however, the real data (from $P^\star$) can be more expensive to collect compared to the synthetic data (from $P^i$), and the independence of $n$ on the complexity of $|\Fcal|$ can be appealing.

\section{Selection of $o_{t+1} | \tau_t, a_t$ (``\OS'')} \label{sec:os}

We now show that the dynamical-system nature of POMDPs adds interesting twists to the problem and allows for alternative solutions. In particular, we show that choosing the right conditional of $s_t|\tau_t$ (by reducing to the problem of selecting the conditional distribution of $Y|X$ with $X = \tau_t$ and $Y= s_t$) is not the only way to select the belief state approximation. Instead, we can select for the right \textit{observable model} induced by the belief state approximations.

\paragraph{Observable Model} A POMDP $\PO$ and an approximate belief state $\bs$ together defines an \textit{observable model} $M_{\PO, \bs}$, which is a mapping in the form of $\Hcal\times\Acal\to\Delta(\Ocal)$. Using this model $M_{\PO, \bs}$, we can efficiently sample the next observation $o_{t+1}$ given an history $\tau_t$ and action $a_t$, as described by the following sampling process: $o_{t+1} \sim M_{\PO, \bs}(\cdot|\tau_t, a_t)$ is equivalent to
\begin{align} \label{eq:obs-sampling}
s_t \sim \bs(\cdot|\tau_t), s_{t+1} \sim P(\cdot|s_t, a_t), o_{t+1} \sim E(\cdot|s_{t+1}).
\end{align}
This model can be equivalently treated as a history-based MDP, as it describes how the next state (length-$(t+1)$ history) can be sampled from the current state-action pair (length-$t$ history and action): given $\tau_t$, $a_t$, the generative process for $\tau_{t+1}$ is simply
$$
o_{t+1} \sim M_{\PO, \bs}(\cdot|\tau_t, a_t), ~~ \tau_{t+1} = \tau_t \circ a_t \circ o_{t+1}, 
$$
where $\circ$ is concatenation. 
It is not hard to see that when $\bs=\bstrue$, $M_{\PO, \bstrue}$ describes the true history-based MDP induced by POMDP $\PO$, i.e., $M_{\PO, \bstrue} = M_{\PO}$. 
Given that most key RL concepts in POMDPs, such as value functions and optimal policies, can be defined through the induced observable model (see Section~\ref{sec:mdp}), a natural idea is to apply the conditional selection procedure in Section~\ref{sec:conditional} but with the following setup:
\begin{itemize}
\item $X = (\tau_t, a_t)$, $Y = o_{t+1}$.
\item $P^\star$ correspond to generating $X, Y$ pairs by sampling trajectories using some behavior policy $\pi_b: \Hcal \to\Delta(\Acal)$. 
\item Each candidate $\bs^\ii{i}$ induces a conditional $P_i(\cdot|X) = M_{\PO, \bs^\ii{i}}(\cdot|\tau_t, a_t)$.
\end{itemize}
Then under Assumption~\ref{asm:nontrivial_F}, we can directly invoke Theorem~\ref{thm:scheffe-TV} for an expected TV guarantee. For example:
\begin{corollary} \label{cor:observable-model}
Bind $X = (\tau_t, a_t)$ and $Y=o_{t+1}$ and define $P^\star$, $\{P_i\}$ as described above. Under the conditions of Theorem~\ref{thm:scheffe-TV}, with probability at least $1-\delta$, we will select a belief state approximation $\bs$, such that (note that $M_{\PO} = M_{\PO, \bstrue})$
\begin{align} \label{eq:observable-guarantee}
\EE_{(\tau_t, a_t) \sim \PO^{\pi_b}}[\TV(M_{\PO, \bs}(\cdot|\tau_t, a_t), M_{\PO}(\cdot|\tau_t, a_t))] \le \epsilon.
\end{align}
\end{corollary}
Furthermore, such a guarantee are immediately implied from (and thus weaker than) that for \lss: 

\begin{proposition} \label{prop:gap}
Fixing any $\tau_t, a_t$, we have
$$
\TV(M_{\PO, \bs}(\cdot|\tau_t, a_t), M_{\PO}(\cdot|\tau_t, a_t)) \le 
\TV(\bs(\cdot|\tau_t), \bstrue(\cdot|\tau_t)).
$$
Therefore, Eq.\eqref{eq:belief-guarantee} immediately implies Eq.\eqref{eq:observable-guarantee}.
\end{proposition}
\begin{proof}
Conditioned $s_t$, the process of generating $o_{t+1}$ is independent of whether $s_t$ is generated from $\bstrue$ or $\bs$, so the claim is a direct consequence of the data processing inequality. 
\end{proof}
If the final goal is to compute objects that solely depend on the observable model $M_{\PO}$, such as value functions or optimal policies, it seems that \os is an equally legitimate solution. In fact, \os has an additional advantage of being more robust to misspecification: %it is easy to construct examples where the inequality in Proposition~\ref{prop:gap} is strict and has a large gap; 
%for example, consider a trivial POMDP where $E(\cdot|s_t)$ has no dependence on $s_t$. Then, an arbitrary $\bs$ can still produce the perfectly accurate observable model $M_{\PO, \bs} = M_{\PO}$ (i.e., LHS is $0$), even when the RHS is arbitrarily large. In less contrived scenarios, 
if the latent state $s_t$ includes dummy variables that do not affect the observable dynamics, \os can effectively ignore the errors of $\bs$ on these inconsequential variables and thus require a weaker version of realizability, while \lss still treats such $\bs$ as incorrect and insists on choosing the groundtruth $\bstrue$. 

This \os approach reflects a prevailing theme in RL research on POMDPs, namely \textit{\textbf{behavioral equivalence}}, that the observable behavior of a POMDP is what ultimately matters, and latent states are \textit{ungrounded} objects and only a convenient way to help express the observable behavior (e.g., the latent state transition $P$ and emission $E$ are a convenient way to parameterize the observable dynamics $M_\PO$). This philosophy is most pronounced in the research of Predictive State Representations \citep{littman2002predictive, singh2004predictive} and minimal information state \citep{subramanian2022approximate}, %\footnote{This is related to the notion of bisimulation abstraction in MDPs \citep{ferns2004metrics, li2006towards}. In fact, the belief MDP (evolution of belief state vectors) of a POMDP can be viewed as the bisimulation abstraction of the history-based MDP. While model minimization in the context of bisimulation also removes the redundancy} 
and is also manifested in recent statistical results for learning in POMDPs \citep{zhang2025statistical} (see also Section~\ref{sec:offline}). As we will see next, however, \os can suffer from surprising degeneracy in downstream use cases when \lss is well-behaved, which seems to contradict the spirit of behavioral equivalence.\footnote{The paradox is resolved by realizing that we rely on $P$ and $E$ for efficient sampling which ground the latent states.}

\section{Roll-Out Guarantees and Temporal Consistency} \label{sec:rollout}
In this section, we investigate how the errors in $\bs$ (measured in the forms of Eq.\eqref{eq:belief-guarantee} or \eqref{eq:observable-guarantee}) can affect the downstream RL tasks. We consider the most basic yet representative use of state resetting enabled by an approximate $\bs$: 

\paragraph{Roll-out:} Sample trajectories to obtain a Monte-Carlo estimate of $Q_{\PO}^\pi(\tau_t, a_t)$ for a given target policy $\pi$. 

This procedure is useful for debugging and simulator selection \citep{sajed2018high, liu2025model}, enables Monte-Carlo control \citep{sutton2018reinforcement}, and forms the basis of Monte-Carlo Tree Search \citep{kocsis2006bandit, silver2010monte, browne2012survey}. 
Perhaps surprisingly, despite the simplicity of the task, there are many nuances to the question. \lss and \os interact in subtle but consequential ways with the roll-out method, and \os can have degenerate behaviors when used with the most natural roll-out procedure. 

\subsection{\ROR Roll-Out}

We now consider what error guarantees can be obtained for estimating $Q_{\PO}^\pi(\tau_t, a_t)$ via Monte-Carlo roll-outs using an approximate belief state $\bs \approx \bstrue$, e.g., one with guarantees established in Theorem~\ref{thm:scheffe-TV}. In particular, the obvious (and \textit{seemingly} unique) roll-out procedure, which we call ``\ROR Roll-out'' (whose meaning will be clear shortly), is: \vspace*{.5em}

\begin{mdframed}
\textbf{\ROR Roll-out.} \quad 
{Input:} $\tau_t, a_t, \pi$. 
\begin{enumerate}
\item Sample $s_t \sim \bs(\cdot|\tau_t)$. 
\item Take given action $a_t$, and generate $s_{t+1} \sim P(\cdot|s_t, a_t)$, $o_{t+1} \sim E(\cdot|s_{t+1})$. 
\item Repeat Step 2 for subsequent time steps by taking actions according to $\pi$, and collect $\sum_{t'=t+1}^H R(o_{t'})$ as a Monte-Carlo return.
\end{enumerate}
\end{mdframed}

Since the error due to finite roll-outs can be easily handled by Hoeffding's inequality, we will focus on the expected value (i.e., assuming infinitely many roll-outs) obtained through the above procedure, denoted as $Q_{\RO(\PO, \bs)}^\pi(\tau_t, a_t)$, and analyze its error relative to the groundtruth $Q_{\PO}^\pi(\tau_t, a_t)$. Concretely, the guarantee in the form of Eq.\eqref{eq:belief-guarantee}, obtained via \lss, immediately implies the proposition below.  
Since our guarantee for $\bs$ in Eq.\eqref{eq:belief-guarantee} is not pointwise, it should not be surprising that the error bound for $Q_{\RO(\PO, \bs)}^\pi$ is also given in an average sense w.r.t.~the distribution of trajectories under $\pi_b$, since that is what we use to train the classifiers and select for $s_t|\tau_t$. 

\begin{proposition} \label{prop:belief-rollout}
If Eq.\eqref{eq:belief-guarantee} (guarantee for \lss) holds for some $t$, then 
\begin{align}
\EE_{(\tau_t, a_t)\sim \PO^{\pi_b}}[|Q_{\RO(\PO, \bs)}^\pi(\tau_t, a_t) - Q_{\PO}^\pi(\tau_t, a_t)|] \le \epsilon \cdot \Vmax.
\end{align}
\end{proposition}
\begin{proof}
The procedure for rolling out using $\bs$ vs.~$\bstrue$ is identical after $s_t$ is sampled, so $\EE[\sum_{t'=t+1}^H R(o_{t'}) | \tau_t, a_t, s_t; \pi]$ is the same for both $\bs$ and $\bstrue$, where the expectation is w.r.t.~the randomness of the \ROR procedure. The final result is just taking its expectation w.r.t.~$s_t \sim \bs(\cdot | \tau_t)$ vs.~$s_t \sim \bstrue(\cdot | \tau_t)$, so
$$
|Q_{\RO(\PO, \bs)}^\pi(\tau_t, a_t) - Q_{\PO}^\pi(\tau_t, a_t)| \le \TV(\bs(\cdot|\tau_t), \bstrue(\cdot|\tau_t)) \cdot \Vmax.
$$
Plugging this into $\EE_{(\tau_t, a_t) \sim \PO^{\pi_b}}[\cdot]$ completes the proof. 
\end{proof}

As a natural follow-up question, can we obtain similar error bounds from Eq.\eqref{eq:observable-guarantee}, which is obtained from \os?

\subsection{\RAR Roll-out} \label{sec:raro}

Recall that \os enjoys the guarantee in Eq.\eqref{eq:observable-guarantee}: 
$$
\EE_{(\tau_t, a_t) \sim \PO^{\pi_b}}[\TV(M_{\PO, \bs}(\cdot|\tau_t, a_t), M_{\PO}(\cdot|\tau_t, a_t))] \le \epsilon.
$$
It is tempting to provide guarantees for the expected roll-out value by directly reducing to existing MDP analysis, using the following argument: note that $M_{\PO, \bs}$ and $M_{\PO,\bstrue}$ are two history-based MDPs.
When we treat history as state, Eq.\eqref{eq:observable-guarantee} is similar to the kind of guarantees obtained from MLE in MDPs,\footnote{The guarantee of MLE has a square on the TV distance, i.e., $\EE[\TV(\cdot)^2] \le \ldots$.} which immediately leads to a policy-evaluation guarantee:

\begin{theorem} \label{thm:model-based}
Suppose two MDPs only differ in the transition dynamics, $P$ vs.~$P'$. In addition, for every $t$, assume
\begin{align} \label{eq:mdp-error}
\EE_{(s_t, a_t) \sim P^{\pi_b}}[\TV({P'}(\cdot|s_t,a_t), P(\cdot|s_t, a_t))] \le \epsilon.
\end{align}
When the roll-out policy $\pi = \pi_b$, it holds that for every $t$,
$$
\EE_{(s_t, a_t) \sim P^{\pi_b}}[|Q_{P'}^{\pi}(s_t, a_t) - Q_{P}^{\pi}(s_t, a_t)|] \le \epsilon (H-t) \Vmax.
$$
\end{theorem}
This result is standard in model-based RL, and we provide its proof in Appendix~\ref{app:rollout} for completeness. It is also possible to extend the result to other roll-out policies $\pi\ne\pi_b$ by paying for a coverage coefficient, which we will discuss in Section~\ref{sec:coverage} but not consider in this section. 
The exact match between Eq.\eqref{eq:mdp-error} and Eq.\eqref{eq:observable-guarantee} immediately implies that:

\begin{corollary} \label{cor:raro}
If Eq.\eqref{eq:observable-guarantee} holds for all $t$, then 
for $\pi = \pi_b$ we have (note that $Q_{\PO}^{\pi} =Q_{M_{\PO, \bstrue}}^{\pi}$):
$$
\EE_{(\tau_t, a_t)\sim \PO^{\pi_b}}[|Q_{M_{\PO, \bs}}^{\pi}(\tau_t, a_t) - Q_{\PO}^{\pi}(\tau_t, a_t)|] \le \epsilon (H-t)\Vmax.
$$
\end{corollary}
While the guarantee here has an additional horizon factor $(H-t)$, it is nevertheless a solid guarantee. The final piece of the puzzle is the observation that \ROR \textit{seems} to be the only way to sample from $\Gamma$ using $\bs$, so it \textit{probably} coincides with $M_{\PO, \bs}$, i.e., $Q_{\RO(\PO, \bs)}^{\pi} = Q_{M_{\PO, \bs}}^{\pi}$ for any $\pi$. 

This intuitive statement, however, is in sharp conflict with the following example:

\begin{example}[$Q_{\RO(\PO, \bs)}^{\pi_b}$ \textbf{cannot} enjoy the guarantee of Corollary~\ref{cor:raro}] \label{ex:counter}
Consider an arbitrary POMDP $\PO$, except that the emission $E(\cdot|\cdot)$ is state-independent at some time step $t_0$. Every candidate $\bs$ predicts the correct belief state $\bstrue(\cdot|\tau_t)$ for any $t\ne t_0$, but can predict arbitrary distributions for $t=t_0$. In this case, the observable behavior (i.e., the induced law of $M_{\PO, \bs}$) of $\bs$ is indistinguishable from that of $\bstrue$ at any time step, so any $\bs$ could be selected. However, a \ROR roll-out starting at $t_0-1$ will generally be incorrect since the arbitrarily generated $s_{t_0-1}$ will have a lingering effect, i.e., producing an incorrect distribution of $s_{t_0}$ and subsequent latent states.
\end{example}

\begin{figure}[t]
\centering
\begin{tabular}{c c}
Real trace:   & X O O $\mid$ O X O O O O X O O X \\
\ROR:   & X O O $\mid$ X O O O O X O O X O \\
\RAR: & X O O $\mid$ X X X X X X X X X X
\end{tabular}
\caption{Toy example for illustrating the difference between \ROR and \RAR. In this binary-observation, action-less system, ``X'' represents the occurrence of an event. Every time an event happens (``X''), the system samples the interval till next event from some distribution. The history of interest is ``X O O'', and the first row shows the real trajectory. $\bs$ always sets the latent state to be $0$, i.e., predicts that next event will occur immediately. \label{fig:queue}}
\end{figure}

We are facing a paradox: on one hand, Example~\ref{ex:counter} shows that \os cannot enjoy the guarantee in Corollary~\ref{cor:raro} for \ROR roll-out; on the other hand, the observable model $M_{\PO, \bs}$, a history-based MDP determined by $\PO$ and $\bs$, does enjoy a standard guarantee for its induced Q-value. Given that \ROR seems to be the only reasonable way to roll-out trajectories using $\PO$ and $\bs$, it is reasonable to believe that such roll-outs correspond to the Q-value of $M_{\PO, \bs}$. So what gives?

\begin{fact}
$Q_{\RO(\PO, \bs)}^\pi$ is \textit{not} equivalent to $Q_{M_{\PO, \bs}}^\pi$. 
\end{fact}

The reason why these two objects are different should be made clear by the following procedure that actually produces roll-outs according to $M_{\PO, \bs}$: \vspace*{.5em}

\begin{mdframed}
\textbf{\RAR Roll-out}. \qquad 
{Input:} $\tau_t, a_t, \pi$.
\begin{enumerate}
\item Sample $s_t \sim \bs(\cdot | \tau_t)$. 
\item Take given action $a_t$, and generate $s_{t+1} \sim P(\cdot|s_t, a_t)$, $o_{t+1} \sim E(\cdot|s_{t+1})$. 
\item \textbf{Replace $s_{t+1}$ with a fresh sample from $\bs(\cdot|\tau_{t+1})$} where $\tau_{t+1} = \tau_t \circ a_t \circ o_{t+1}$.
\item Repeat Steps 2 and 3 for subsequent time steps by taking actions according to $\pi$, and collect $\sum_{t'=t+1}^H R(o_{t'})$ as a Monte-Carlo return.
\end{enumerate}
\end{mdframed}

\ROR and \RAR are equivalent if $\bs = \bstrue$ but are generally different; see Figure~\ref{fig:queue} for an example. 
We conclude this section with the following remarks that reconcile the earlier paradox:
\begin{itemize}[leftmargin=*]
\item The main difference between \ROR and \RAR is that, at any time $t$, the observable trajectory $\tau_{t}$ is a sufficient statistics for simulating the rest of the trajectory in \RAR. For \ROR, however, the sufficient statistic is $(\tau_{t}, s_{t})$. 
Therefore, when we only have \os guarantee  but not that of \lss, \ROR is only guaranteed to produce the correct $o_{t+1}$, but not future observations due to the lingering effect of possibly wrong $s_t$. 
\item \Os  can still enjoy the guarantee in Corollary~\ref{cor:observable-model} via \RAR roll-out, but it suffers from an additional horizon factor compared to Proposition~\ref{prop:belief-rollout} due to the repeated application of the inaccurate $\bs$ (c.f.~Figure~\ref{fig:queue}). On a related note, Proposition~\ref{prop:belief-rollout} only requires Eq.\eqref{eq:belief-guarantee} to hold for the $t$ that is the subscript of the $\tau_t$ we start roll-out from, but Corollary~\ref{cor:observable-model} requires Eq.\eqref{eq:observable-guarantee} to hold for all $t'\ge t$. Moreover, when the roll-out policy $\pi \ne \pi_b$, Corollary~\ref{cor:raro} needs to pay for an additional coverage coefficient (see Section~\ref{sec:coverage}) while Proposition~\ref{prop:belief-rollout} does not. Therefore, while \lss + \RAR can also enjoy Corollary~\ref{cor:observable-model} via Proposition~\ref{prop:gap}, it is inferior to \lss + \ROR in both error propagation and computational efficiency.
\item As another possible misconception, it may be tempting to think that the observable dynamics of \ROR is $M_{\PO, \bs}$ at time step $t$ and $M_{\PO}$ for subsequent time steps, since all later simulations do not involve the use of the inaccurate $\bs$ and hence should be ``correct''. This is not true due to, again, the lingering effect of wrong $s_t$ distribution. (In fact, if this held, Corollary~\ref{cor:observable-model} would hold for $Q_{\RO(\PO, \bs)}^\pi$ without the horizon factor.) That is, although all subsequent simulations seem correct, the induced conditional law of the observables, $o_{t'+1}|\tau_{t'},a_{t'}$ for $t' > t$ is \textit{not} the same as the true $M_{\PO}$.
\item As potential future directions, it will be interesting to consider if the inconsistency between \ROR and \RAR can be turned into a method for selecting belief-state approximations, and whether the insights can be used to reduce the error accumulation of \os. 
\end{itemize}

\begin{table}[t!]
\centering
\caption{Relationship between the selection methods and the roll-out methods. \label{tab:comparison}}
\begin{tabular}{c|c|c}
 & \LSS & \OS                 \\ \hline
\makecell[c]{Agnostic to redundant\\[-2.5pt] latent variables} & \redcross & \greencheck (see the end of Section~\ref{sec:os}) \\
\hline
\ROR & \greencheck ~ Proposition~\ref{prop:belief-rollout}  & \redcross ~ Example~\ref{ex:counter}  \\
\hline
\RAR & \multicolumn{2}{c}{\greencheck ~  Corollary~\ref{cor:observable-model} (worse than Proposition~\ref{prop:belief-rollout})}                  
\end{tabular}
\end{table}

\subsection{Case Study: Simulator Selection from Real-System Traces} \label{sec:offline}

The previous sections may leave the impression that \lss is more superior to \os other than a minor disadvantage regarding redundant latent variables. Below we study a motivating scenario mentioned earlier, where it is beneficial to integrate both \lss and \os into the solution and  exploit the strength of each. 

\paragraph{Problem Setup.} Consider the problem of learning from real-system data. Let $\POreal$ be a real system, from which we draw observable trajectories with policy $\pi_b$. The goal is to use these data trajectories to select among candidate simulators $\{\PO_k\}_{k=1}^K$ that best matches the dynamics of the real system, and as before we assume realizability i.e., $\POreal \in \{\PO_k\}$. In a way, this is essentially a model estimation problem in POMDPs, and the standard approach (as mentioned earlier in Section~\ref{sec:os}) is MLE \citep{liu2022optimistic}:  
$$
\argmin_{k} \sum_{j} \log M_{\PO_k}(o_{t+1}^\ii{j} | \tau_t^\ii{j}, a_t^\ii{j}),
$$
where all variables with $(j)$ subscript come from the $j$-th  data trajectory.

Unfortunately, this solution is not directly applicable to our setting, since we do not assume probability mass/density access to $P$ or $E$ in any given simulator $\PO$ and thus cannot compute $M_{\PO_k}(\cdot|\cdot)$. Instead, we can directly leverage our \os (Section~\ref{sec:os}): while it is initially designed to select $M_{\PO}$ from $\{M_{\PO, \bs}: \bs \in \Bcal\}$, the procedure and analyses immediately apply to the problem here where we select $M_{\PO_\star}$ from $\{M_{\PO_k}\}$. 

But that brings a further problem: \os requires efficient sampling access to the candidate conditionals, which is $M_{\PO}(o_{t+1}|\tau_t, a_t)$ for $\PO \in \{\PO_k\}$ here. That requires having sampling access to the belief state in $\PO$ (Eq.\eqref{eq:obs-sampling}), which is not available. 
However, that is exactly the problem we have been dealing with so far! So let's assume that we are given $\Bcal$,\footnote{The candidate set $\Bcal$ can be designed for each $\PO \in \{\PO_k\}$ separately, but we assume $\Bcal$ is the same across candidate simulators for ease of presentation.} such that for each $\PO \in \{\PO_k\}$, the true belief state $\bstrue \in \Bcal$.\footnote{We still use $\bstrue$ to refer to the true belief state of the simulator $\PO$ under consideration. Note that we do not need to refer to the belief state of the real system $\POreal$ and hence does not give it a notation.} This leads to a two-stage procedure: in \textbf{Stage 1}, for each $\PO$, we select $\bs_\PO \in\Bcal$ as its belief-state approximation; in \textbf{Stage 2}, we use \os to select an observable model from $\{(\PO_k, \bs_{\PO_k})\}$. 

\paragraph{Solution Solely based on \OS}
Ultimately, we need to select a $(\PO, \bs)$ pair from $\{\PO_k\} \times \Bcal$. Given that observable trajectories from $\POreal$ force the use of \os in the second stage, a natural simplification is to lump the first stage into the second and solve both simultaneously with one unified \os instance. That is, we select
$$
M_{\POreal} \quad \textrm{from} \quad \{M_{\PO, \bs}: \PO \in \{\PO_k\}, \bs\in\Bcal\}.
$$
Invoking the analyses in Corollaries~\ref{cor:observable-model} and \ref{cor:raro}, we immediately have:\footnote{We relax $(H-t)$ in Corollary~\ref{cor:raro} to $H$ here for readability.} for the selected $(\PO, \bs)$ pair,
\begin{align}
&~ \EE_{(\tau_t, a_t) \sim \POreal^{\pi_b}}[\TV(M_{\PO, \bs}(\cdot|\tau_t, a_t), M_{\POreal}(\cdot|\tau_t, a_t))] \le \epsilon_0, \\ \label{eq:os-only-rar}
&~ \EE_{(\tau_t, a_t)\sim \POreal^{\pi_b}}[|Q_{M_{\PO, \bs}}^{\pi_b}(\tau_t, a_t) - Q_{\POreal}^{\pi_b}(\tau_t, a_t)|] \le \epsilon_0 H \Vmax,
\end{align}
where $\epsilon_0$ can be determined by the number of data trajectories from $\POreal$ through the sample-complexity statement in Theorem~\ref{thm:scheffe-TV}.\footnote{Concretely, to achieve $\epsilon_0$ error, the number of trajectories needed is $O(\log(mK/\delta)/\alpha^2 \epsilon_0^2$, where $m$ and $K$ are the sizes of $\Bcal$ and $\{\PO_k\}_{k=1}^K$, respectively.} 

The problem is, if we want to roll-out trajectories using the selected $(\PO, \bs)$ to approximate $Q_{\POreal}^{\pi_b}$, the only valid approach is \RAR (Eq.\eqref{eq:os-only-rar}), and \ROR will not enjoy any guarantee given Example~\ref{ex:counter}. However, \RAR is computationally costly especially when sampling from $\bs$ has a nontrivial cost, and in practice \RAR is often a bad idea given repeated injection of the error of $\bs \ne \bstrue$ (Figure~\ref{fig:queue}). This begs the the question: can we enjoy a guarantee similar to Eq.\eqref{eq:os-only-rar} while rolling out from the selected $(\PO, \bs)$ using \ROR?

\paragraph{Two-Stage Solution} We now show that the natural two-stage solution achieves the goal if we use \lss in the first stage (selecting $\bs$ for $\PO$). The analysis turns out to be somewhat nontrivial, which we provide below:

\begin{theorem} \label{thm:2-stage}
Assume that the selected $(\PO, \bs)$ satisfies
\begin{align} \label{eq:stage-1-lss}
&~ \EE_{\tau_t \sim \PO^{\pi_b}}[\TV(\bs(\cdot|\tau_t), \bstrue(\cdot|\tau_t))] \le \epsilon_1. \\ \label{eq:os-with-stage-1-lss}
&~ \EE_{(\tau_t, a_t) \sim \POreal^{\pi_b}}[\TV(M_{\PO, \bs}(\cdot|\tau_t, a_t), M_{\POreal}(\cdot|\tau_t, a_t))] \le \epsilon_0.
\end{align}
Then
$$
\EE_{(\tau_t, a_t) \sim \PO_{\star}^{\pi_b}}[|Q_{\RO(\PO, \bs)}^{\pi_b}(\tau_t, a_t) - Q_{\PO_\star}^{\pi_b}(\tau_t, a_t)|] \le  (2\epsilon_0 + 3\epsilon_1) H \Vmax.
$$
\end{theorem}
The conditions of the theorems are the guarantees of \lss in Stage 1 (Eq.\eqref{eq:belief-guarantee}) and \os in Stage 2 (Eq.\eqref{eq:observable-guarantee} when $\POreal$ is the groundtruth).\footnote{There is a slight caveat in Stage 2's guarantee due to the violation of realizability, i.e., after stage 1, the true belief state of $\PO_{k^\star} = \POreal$ might have been eliminated. See Appendix~\ref{app:violate} for further discussions. \label{ft:violate}} The final guarantee resembles Eq.\eqref{eq:os-only-rar}, except that it permits the use of \ROR roll-out. The error bound is slightly worse than Eq.\eqref{eq:os-only-rar} by a multiplicative constant and an additional dependence on $\epsilon_1$. However, note that $\epsilon_0$ is determined by the amount of real-system data which often is fixed, while $\epsilon_1$ is determined by the amount of data sampled from each simulator $\PO$. So overall the guarantee is still comparable to Eq.\eqref{eq:os-only-rar}. 

\begin{proof}[Proof of Theorem~\ref{thm:2-stage}]
Eq.\eqref{eq:stage-1-lss} implies that $\EE_{(\tau_t, a_t) \sim \PO^{\pi_b}}[\TV(M_{\PO, \bs}(\cdot|\tau_t, a_t), M_{\PO}(\cdot|\tau_t, a_t))] \le \epsilon_1$. Both this inequality and Eq.\eqref{eq:os-with-stage-1-lss}, through the sub-additivity of TV distance for product distributions, implies
$$
\TV(M_{\PO, \bs}^{\pi_b}, \PO_\star^{\pi_b}) \le \epsilon_0 H, 
\qquad~ \TV(M_{\PO, \bs}^{\pi_b}, {\PO}^{\pi_b}) \le \epsilon_1 H. 
$$
Therefore, $\TV({\PO}^{\pi_b}, \PO_\star^{\pi_b}) \le (\epsilon_0 + \epsilon_1) H.$ Next, we have
$$
\EE_{(\tau_t, a_t) \sim \PO^{\pi_b}}[|Q_{\RO(\PO, \bs)}^{\pi_b}(\tau_t, a_t) - Q_{M_{\PO, \bs}}^{\pi_b}(\tau_t, a_t)|] \le \epsilon_1 (1 + H) \Vmax,
$$
because we can use $Q_{\PO}^{\pi_b}$ as the bridge term, and both $Q_{\RO(\PO, \bs)}^{\pi_b}$ and $Q_{M_{\PO, \bs}}^{\pi_b}$ are close to it thanks to Proposition~\ref{prop:belief-rollout} and Corollary~\ref{cor:raro}. On the other hand, Eq.\eqref{eq:os-with-stage-1-lss} enables the following through Corollary~\ref{cor:raro}:
$$\EE_{(\tau_t, a_t) \sim \POreal^{\pi_b}}[|Q_{M_{\PO, \bs}}^{\pi_b}(\tau_t, a_t) - Q_{\PO_\star}^{\pi_b}(\tau_t, a_t)|] \le \epsilon_0 H \Vmax.$$
Finally, putting everything together:
\begin{align*}
&~ \EE_{(\tau_t, a_t) \sim \PO_{\star}^{\pi_b}}[|Q_{\RO(\PO, \bs)}^{\pi_b}(\tau_t, a_t) - Q_{\PO_\star}^{\pi_b}(\tau_t, a_t)|] \\
\le &~ \EE_{(\tau_t, a_t) \sim \PO_{\star}^{\pi_b}}[|Q_{\RO(\PO, \bs)}^{\pi_b}(\tau_t, a_t) - Q_{M_{\PO, \bs}}^{\pi_b}(\tau_t, a_t)|] + \epsilon_0 H \Vmax \\
\le &~ \EE_{(\tau_t, a_t) \sim \PO^{\pi_b}}[|Q_{\RO(\PO, \bs)}^{\pi_b}(\tau_t, a_t) - Q_{M_{\PO, \bs}}^{\pi_b}(\tau_t, a_t)|] + (2\epsilon_0 + \epsilon_1) H \Vmax \\
\le &~ (2\epsilon_0 + 3\epsilon_1) H \Vmax,
\end{align*}
where the third line changes the distribution from $\POreal^{\pi_b}$ to $\PO^{\pi_b}$ by paying the for the TV-distance multiplied by the boundedness of the function.  
\end{proof}

\section{Further Discussions} \label{sec:discussion}

\subsection{Generalization to Other Sampling Distributions} \label{sec:coverage}

The roll-out guarantees in Theorem~\ref{thm:scheffe-TV} and Corollary~\ref{cor:observable-model} all consider Q-function errors under the distribution $(\tau_t, a_t)\sim \PO^{\pi_b}$, under which we train the classifiers to select the belief-state approximation, and Corollary~\ref{cor:observable-model} further restricts the roll-out policy to be $\pi = \pi_b$. Naturally, one would wonder what happens when the error is measured under a different sampling distribution (e.g., the occupancy induced from a different roll-in policy $\pi'$), and when $\RAR$ is given a general roll-out policy $\pi \ne \pi_b$. These questions are well-understood in the MDP literature (especially offline RL theory) that we can pay some form of \textit{coverage coefficient} to translate the error from one distribution to another:

\begin{proposition}
Consider an MDP with transition $P$. 
\begin{enumerate}
\item (\ROR, extension of Proposition~\ref{prop:belief-rollout}) Given any $Q: \Scal_t\times\Acal\to \RR$ where  $\EE_{(s_t, a_t) \sim P^{\pi_b}}[|Q(s_t, a_t) - Q_{P}^{\pi}(s_t, a_t)|] \le \epsilon$ for some fixed $t$, for any roll-in policy $\pi'$,
\begin{align} \label{eq:ro-coverage}
\EE_{(s_t, a_t) \sim P^{\pi'}}[|Q(s_t, a_t) - Q_{P}^{\pi}(s_t, a_t)|] \le \epsilon \cdot \left\|\oc_t^{\pi'} / \oc_t^{\pi_b} \right\|_\infty ,
\end{align}
where $\oc_t^{(\cdot)}$ is the marginal distribution (a.k.a.~occupancy) of $(s_t, a_t)$ induced by a policy in $P$. 
\item (\RAR, extension of Proposition~\ref{cor:observable-model}) Consider another MDP with transition $P'$ where Eq.\eqref{eq:mdp-error} holds for all $t$. Let $(\pi')^t \circ (\pi)^{t'-t}$ be the policy that follows $\pi'$ for the first $t$ steps and $\pi$ for the next $t'-t$ steps, then
$$
\EE_{(s_t, a_t) \sim P^{\pi'}}[|Q_{P'}^{\pi}(s_t, a_t) - Q_{P}^{\pi}(s_t, a_t)|] \le \epsilon \cdot  \sum_{t'= t}^{H-1}  \left\|\oc_t^{(\pi')^t \circ (\pi)^{t'-t}} / \oc_{t'}^{\pi_b} \right\|_\infty.
$$
\end{enumerate}    
\end{proposition} 
These results can be directly applied to POMDPs by mapping state $s_t$ in the proposition to the observable history $\tau_t$, $P$ to the observable dynamics of $\PO$ (i.e., $M_{\PO, \bstrue}$), $Q$ to $Q_{\RO(\PO, \bs)}^{\pi}$, and $P'$ to $M_{\PO, \bs}$. However, while the coverage coefficients that appear in the proposition are often acceptable in MDPs, their behaviors are not as benign in POMDPs: for example, Eq.\eqref{eq:ro-coverage} becomes
\begin{align} \label{eq:ro-coverage-pomdp}
\EE_{(\tau_t, a_t) \sim \PO^{\pi'}}[|Q_{\RO(\PO, \bs)}(\tau_t, a_t) - Q_{\PO}^{\pi}(\tau_t, a_t)|] \le \epsilon \cdot \left(\max_{\tau_t, a_t}\frac{\PP_{\PO}^{\pi'}[\tau_t, a_t]}{\PP_{\PO}^{\pi_b}[\tau_{t}, a_{t}]}\right) ,
\end{align}
where $\PP_{\PO}^{\pi'}[\tau_t, a_t]$ is the probability assigned to the partial trajectory $(\tau_t, a_t)$ in $\PO$ under $\pi'$ as the sampling policy, and 
$$
\frac{\PP_{\PO}^{\pi'}[\tau_t, a_t]}{\PP_{\PO}^{\pi_b}[\tau_{t}, a_{t}]} = \prod_{t'=0}^{t} \frac{\pi'(a_{t'}|\tau_{t'})}{\pi_b(a_{t'}|\tau_{t'})}
$$
is the infamous cumulative product of importance weights found in importance sampling \citep{precup2000eligibility}. 
This is actually a general problem whenever we apply MDP results to POMDPs via a reduction to history-based MDPs, and circumventing it often requires algorithms and coverage concepts specifically designed for POMDPs \citep{zhang2024curses}. It remains an interesting question whether those ideas (such as the notion of belief \& outcome coverage proposed by \citet{zhang2024curses}) are useful for the belief-state selection problem considered in this paper.

\paragraph{Generalization via Sufficient Statistics} A mitigation to the above problem is to make and leverage structural assumptions on $\bs$. In particular, we may assume that $\bs(\cdot|\tau_t)$ is generated via a two-stage procedure: 
$$
s_t\sim \bs(\cdot|\tau_t) \Leftrightarrow z_t = \phi_{\bs}(\tau_t), s_t = \bs(\cdot|z_t).
$$
That is, we first compute the \textit{sufficient statistic} of $\tau_t$ via a function $\phi_{\bs}$, and then sample $s_t$ conditioned on $z_t$. (With a slight abuse of notation we reuse $\bs$ for the conditional distribution in the second stage.) As a starter, if $z_t$ is a discrete variable and the correct $\phi_{\bstrue}$  is known, i.e., $\phi_{\bs} = \phi_{\bstrue}~ \forall \bs\in\Bcal$, we can improve the guarantee of \lss in Eq.\eqref{eq:ro-coverage-pomdp} to the following (see proof in Appendix~\ref{app:ss}): 
$$
\EE_{(\tau_t, a_t) \sim \PO^{\pi'}}[|Q_{\RO(\PO, \bs)}^\pi(\tau_t, a_t) - Q_{\PO}^{\pi}(\tau_t, a_t)|] \le \epsilon \cdot \left(\max_{z_t}\frac{\PP_{\PO}^{\pi'}[z_t]}{\PP_{\PO}^{\pi_b}[z_t]}\right).
$$
Therefore, if $z_t$ takes on a small number of values, there is hope that $\pi_b$ may induce an exploratory distribution over $z_t$ and covers the distribution under $\pi'$. The result can also be easily extended to the case of unknown $\phi_{\bstrue}$ (i.e., $\phi_{\bs}$ can be different for each $\bs$), as we can simply replace $z_t$ in the above bound with the pair $(\phi_{\bs}(\tau_t), \phi_{\bstrue}(\tau_t))$. In this case, the bound requires $\pi_b$ to induce an exploration \textit{joint} distribution over the pair of statistics. For continuous-valued $\phi_{\bs}$, favorable coverage guarantees might still be obtainable if structural assumptions are imposed on the $\bs(\cdot|\phi_{\bs})$ process.

\paragraph{Task-specific Approaches} 
Another route to circumvent the issue related to coverage is to take approaches specific to the task at hand. As an example, for the most basic task of policy evaluation (i.e., estimating $Q_{\PO}^\pi$ for a given $\pi$), there are simple regression based methods\footnote{That is, we can generate trajectories in $\PO$ with $\pi'$ roll-in at time step $t$ and $\pi$ roll-out, and split each trajectory into a regression data point $(\tau_t, a_t) \mapsto \sum_{t'\ge t} r_{t'}$. $Q_{\RO(\PO, \bs)}^\pi$ and $Q_{M_{\PO, \bs}}^\pi$ are treated as candidate regressors, and the true $Q_{\PO}^\pi$ has the least mean squared error.} and selection algorithms based on estimating some variant of the Bellman error. For the latter, the coverage guarantee often does not depend on the coverage in the original MDP, but in an MDP compressed through a low-dimensional representation related to the candidate Q-functions \citep{xie2020batch, zhang2021towards, liu2025model}. Such a deviation from the original dynamics may be a desirable property for POMDPs when the coverage in the original dynamics is not well-behaved.

\subsection{The Choice of $\pi_b$: How to Collect Data}
For most part of the paper, we assume the $\pi_b$, which is used to collect the data needed for the selection of the conditional distributions (Section~\ref{sec:conditional}), is given. In practice, the choice of $\pi_b$ is an important hyperparameter with nuanced effects, which we already had a glimpse in Section~\ref{sec:offline}: while we choose to select belief-state approximations in each simulator $\PO$ by using the same policy $\pi_b$ as the one used to sample the real-system data, the analysis needs to handle the mismatch between the roll-in distributions of $\PO^{\pi_b}$ and $\POreal^{\pi_b}$, which shows up in the final error bound. While this mismatch is shown to be controlled by $(\epsilon_0+\epsilon_1)$, there is the possibility of using a different roll-in policy $\pi_{b'}$ in simulators such that $\PO^{\pi_b'}$ may be a better approximation of $\POreal^{\pi_b}$ than $\PO^{\pi_b}$ is. The issue is further complicated when there is misspecification in $\{\PO_k\}$ and $\Bcal$, which we leave for future investigation. 

Another important motivating scenario for selecting $\bs$ is to use it for learning a good policy in the simulator. In this case, we want $\bs$ to be accurate not just under some fixed distributions, but throughout the learning process when we explore using different policies. A na\"ive approach is to separately optimize one policy for each candidate $\bs \in \Bcal$, and rolling out these policies in the simulator to find the best performing one. However, given the computational intensity of policy optimization, an interesting question is whether we can adjust the choice of $\bs$ as policy optimization unfolds and avoid performing $m=|\Bcal|$ separate policy optimization processes. 

\section*{Acknowledgments}
The author thanks Akshay Krishnamurthy for valuable discussions on early ideas of the project, Sivaraman Balakrishnan for pointers to relevant work on conditional density estimation, 
and Preetum Nakkiran and Sam Power for helpful discussions and suggestions related to Bayesian inference.

\bibliography{RL}

\appendix

\newpage

\section{Proof of Section~\ref{sec:conditional}} \label{app:scheffe}

\begin{proof}[Proof of Theorem~\ref{thm:scheffe-TV}]
In the proof it suffices to only consider the comparison when $i^\star \in \{i, k\}$. 
Under standard concentration argument, the excess risk of ERM on 0/1 classification can be bounded. That is, with the  $N$ given in the theorem statement, under the high-probability event we have the following (the union  bound is reflected by the logarithmic dependence on $n, m, |\Fcal|$ in the sample complexity): let $\epsilon' = \alpha\epsilon$; for any $i\ne i^\star$, 
\begin{align} \label{eq:concentration-classifier}
\textup{acc}_{X_j}^{i, i^\star}(\hat f_j^{i, i^\star}) \ge  \textup{acc}_{X_j}^{i, i^\star}(f_{X_j}^{i, i^\star}) - \epsilon'/24. 
\end{align}
We now link $\textup{acc}_{X}^{i, i^\star}(f)$ to the discrimination power of $f$ w.r.t.~$P_i$ and $P^\star$: given that $f$ has binary output, we have $\Pr[f = 1] = \EE[f]$, and 
\begin{align} \label{eq:link_acc}
\textup{acc}_{X}^{i, i^\star}(f) = &~ 1/2\cdot \left(\Pr_{Y\sim P^\star(\cdot|X)}[f(Y) = 1] + 1 - \Pr_{Y \sim P_i(\cdot|X)}[f(Y) = 1] \right) \\
= &~ 1/2 + 1/2 \cdot \left(\EE_{Y\sim P^\star(\cdot|X)}[f(Y)] - \EE_{Y\sim P_i(\cdot|X)}[f(Y)]\right).
\end{align}
Replacing $\textup{acc}_{X}^{i, i^\star}(f)$ in Eq.\eqref{eq:concentration-classifier} with the above expression, we have $\forall j$,
\begin{align} 
&~ \EE_{Y\sim P^\star(\cdot|X_j)}[\hat f_j^{i, i^\star}(Y)] - \EE_{Y\sim P_i(\cdot|X_j)}[\hat f_j^{i, i^\star}(Y)] \\ \label{eq:excess-risk-ipm}
\ge &~ \EE_{Y\sim P^\star(\cdot|X_j)}[f_{X_j}^{i, i^\star}(Y)] - \EE_{Y\sim P_i(\cdot|X_j)}[f_{X_j}^{i, i^\star}(Y)] - \epsilon'/12. 
\end{align}
Now we consider the concentration of empirical averages to their (conditional) expectations in the final scoring rule in Eq.\eqref{eq:scheffe} when $i^\star \in \{i, k\}$: Conditioned on $\{X_j\}_{j=1}^n$ and all the synthetic data drawn to train the classifiers (which are independent to the randomness of drawing $Y_j$ given $X_j$), $\frac{1}{n} \sum_{j} \hat f_j^{i, k}(Y_j)$ is the average of independent (but generally not identically distributed) random variables, each of which has conditional mean  $\EE_{Y\sim P^\star(\cdot|X_j)}[\hat f_j^{i, k}(Y)]$. Therefore, by Hoeffding's inequality, the $n$ in the statement enables that
\begin{align} \label{eq:concentration-conditional}
\Big|\frac{1}{n} \sum_j \hat f_j^{i, k}(Y_j) - \frac{1}{n} \sum_j \EE_{Y\sim P^\star(\cdot|X_j)}[\hat f_j^{i, k}(Y)] \Big| \le \epsilon'/12.
\end{align}
The same argument holds for $\frac{1}{n} \sum_j \hat f_j^{i,k}(Y_j^{(i)})$ since $Y_j^{(i)}$ is holdout data not used in training:
\begin{align} \label{eq:concentration-model-conditional}
\Big|\frac{1}{n} \sum_j \hat f_j^{i, k}(Y_j^{(i)}) - \frac{1}{n} \sum_j \EE_{Y\sim P_i(\cdot|X_j)}[\hat f_j^{i, k}(Y)] \Big| \le \epsilon'/12.
\end{align}
We now consider the final score in Eq.\eqref{eq:scheffe} when $i^\star \in \{i, k\}$. First consider $i=i^\star$: the two averages share the same mean, so the difference is always bounded by $\epsilon'/6$ given the concentration bounds above regardless of $k$. In the second case, $i\ne i^\star$, $k=i^\star$, where the two averages have different means. Using the concentration bounds above, we have
\begin{align}
&~ \Big|\frac{1}{n}\sum_{j} \hat f_j^{i, i^\star}(Y_j) - \frac{1}{n} \sum_j \hat f_j^{i,i^\star}(Y_j^{(i)})\Big| \label{eq:i-score} \\ 
\ge &~ \frac{1}{n}\sum_{j} \hat f_j^{i, i^\star}(Y_j) - \frac{1}{n} \sum_j \hat f_j^{i,i^\star}(Y_j^{(i)}) \\
\ge &~ \frac{1}{n} \sum_j \EE_{Y\sim P^\star(\cdot|X_j)}[\hat f_j^{i, i^\star}(Y)] - \frac{1}{n} \sum_j \EE_{Y\sim P_i(\cdot|X_j)}[\hat f_j^{i, i^\star}(Y)] - \epsilon'/6. 
\tag{Eqs.\eqref{eq:concentration-conditional} and \eqref{eq:concentration-model-conditional}}  %\\
%\ge &~ \frac{1}{n} \sum_j \EE_{Y\sim P^\star(\cdot|X_j)}[f_{X_j}^{i, i^\star}(Y)] -  \frac{1}{n} \sum_j \EE_{Y\sim P_i(\cdot|X_j)}[f_{X_j}^{i, i^\star}(Y)] - 3\epsilon'/4.
% \tag{Eq.\eqref{eq:excess-risk-ipm}}
\end{align}
For the final $\hat i$ being selected, if $\hat i \ne i^\star$, it must be the case that its score is less than that of $i^\star$ which is at most $\epsilon'/6$, so Eq.\eqref{eq:i-score} for $i=\hat i$ is at most $\epsilon'/6$, and hence
\begin{align} \label{eq:hat-i}
\frac{1}{n} \sum_j \EE_{Y\sim P^\star(\cdot|X_j)}[\hat f_j^{\hat i, i^\star}(Y)] - \frac{1}{n} \sum_j \EE_{Y\sim P_i(\cdot|X_j)}[\hat f_j^{\hat i, i^\star}(Y)] \le \epsilon' / 3. 
\end{align}
Combine this with Eq.\eqref{eq:excess-risk-ipm}, we have
$$
\frac{1}{n} \sum_j \EE_{Y\sim P^\star(\cdot|X_j)}[ f_{X_j}^{\hat i, i^\star}(Y)] - \frac{1}{n} \sum_j \EE_{Y\sim P_i(\cdot|X_j)}[f_{X_j}^{\hat i, i^\star}(Y)] \le 5\epsilon' / 12. 
$$
The next step is to show concentration for the LHS of the above expression. Note that each term like  $\EE_{Y\sim P^\star(\cdot|X_j)}[f_{X_j}^{i, i^\star}(Y)]$ is a non-random property of $X_j$ (we need to union bound over $i$ so that it applies to $i=\hat i$), so their average concentrates to the population expectation with $\epsilon'/12$ error under the $n$ in the theorem statement. Putting together,
\begin{align}
\EE_{X \sim P^\star}\left[\EE_{Y\sim P^\star(\cdot|X)}[f_X^{\hat i, i^\star}(Y)] - \EE_{Y\sim P_i(\cdot|X)}[f_X^{\hat i, i^\star}(Y)]\right] \le  \epsilon'/2.
\end{align}
Finally,
\begin{align}
2 \EE_{X\sim P^\star}[\textup{acc}_X^{\hat i, i^\star}(f_X^{\hat i, i^\star})] -1 
= &~ \EE_{X \sim P^\star}\left[\EE_{Y\sim P^\star(\cdot|X)}[f_X^{\hat i, i^\star}(Y)] - \EE_{Y\sim P_i(\cdot|X)}[f_X^{\hat i, i^\star}(Y)]\right] \tag{Eq.\eqref{eq:link_acc}} \\
\le &~  \epsilon'/2.
\end{align}
Given Assumption~\ref{asm:nontrivial_F}, we have
\begin{align}
1/2 + \epsilon'/4 \ge \EE_{X \sim P^\star} \left[\textup{acc}_X^{\hat i, i^\star}(f_X^{\hat i, i^\star}) \right] \ge 1/2 + \alpha \cdot \Ecal(\hat i, i^\star),
\end{align}
so $\Ecal(\hat i, i^\star) \le \epsilon'/4\alpha = \epsilon/2$. The proof is concluded by noticing that $$\Ecal(\hat i, i^\star) = \nicefrac{1}{2} \cdot \EE_{X \sim P^\star}[\TV(P_{\hat i}(\cdot|X), P_{i^\star}(\cdot|X))]. $$
\end{proof}

\section{Proof of Section~\ref{sec:rollout}}
\label{app:rollout}

\begin{proof}[Proof of Theorem~\ref{thm:model-based}]
To avoid dealing with the last time step separately, we take the convention that any notion of value function evaluates to $0$ at $t=H$ since there is no future reward afterwards. Then for any $t < H$, when $\pi= \pi_b$,
\begin{align*}
&~ \EE_{(s_t, a_t) \sim P^{\pi_b}}\big[\big|Q_P^\pi(s_t, a_t) - Q_{P'}^\pi(s_t, a_t)\big|\big] \\
= &~ \EE_{(s_t, a_t) \sim P^{\pi_b}}\big[\big|E_{s_{t+1} \sim P(\cdot|s_t, a_t)}[R(s_{t+1}) + V_P^\pi(s_{t+1})] - E_{s_{t+1} \sim P'(\cdot|s_t, a_t)}[R(s_{t+1}) + V_{P'}^\pi(s_{t+1})]\big|\big] \\
\le  &~ \EE_{(s_t, a_t) \sim P^{\pi_b}}\big[\big|E_{s_{t+1} \sim P(\cdot|s_t, a_t)}[R(s_{t+1}) + V_{P'}^\pi(s_{t+1})] - E_{s_{t+1} \sim P'(\cdot|s_t, a_t)}[R(s_{t+1}) + V_{P'}^\pi(s_{t+1})]\big|  \\
& \quad + \big|E_{s_{t+1} \sim P(\cdot|s_t, a_t)}[V_{P}^\pi(s_{t+1}) - V_{P'}^\pi(s_{t+1})]\big|\big] \\
\le &~ \EE_{(s_t, a_t) \sim P^{\pi_b}}[\TV(P(\cdot|s_t, a_t), P'(\cdot|s_t, a_t)) \cdot \Vmax] \\
& \quad + \EE_{(s_t, a_t) \sim P^{\pi_b}}\big[E_{s_{t+1} \sim P(\cdot|s_t, a_t)}[\big|V_{P}^\pi(s_{t+1}) - V_{P'}^\pi(s_{t+1})\big|]\big] \\
\le &~  \epsilon \Vmax + \EE_{s_{t+1} \sim P^{\pi_b}}[\big|Q_{P}^\pi(s_{t+1}, \pi) - Q_{P'}^\pi(s_{t+1}, \pi)\big|] \\
\le &~ \epsilon \Vmax + \EE_{(s_{t+1}, a_{t+1}) \sim P^{\pi_b}}[\big|Q_{P}^\pi(s_{t+1}, a_{t+1}) - Q_{P'}^\pi(s_{t+1}, a_{t+1})\big|],
\end{align*}
where the last step uses the fact that $\pi = \pi_b$, and inductively expanding the analysis till the end proves the theorem statement. 
\end{proof}

\section{Proof of Section~\ref{sec:discussion}} \label{app:ss}
\begin{proposition}
If Eq.\eqref{eq:belief-guarantee} holds for some $t$ and all $\bs\in\Bcal$ share the same sufficient statistics $\phi$, then given any roll-in policy $\pi'$, we have
$$
\EE_{(\tau_t, a_t) \sim \PO^{\pi'}}[|Q_{\RO(\PO, \bs)}^\pi(\tau_t, a_t) - Q_{\PO}^{\pi}(\tau_t, a_t)|] \le \epsilon \cdot \left(\max_{z_t}\frac{\PP_{\PO}^\pi[z_t]}{\PP_{\PO}^{\pi_b}[z_t]}\right).
$$
\end{proposition}
\begin{proof}
Following the proof of Proposition~\ref{prop:belief-rollout}, we know that the LHS is controlled by 
$$
\EE_{(\tau_t, a_t)\sim \PO^{\pi'}}[\TV(\bs(\cdot|\tau_t), \bstrue(\cdot|\tau_t))] = \EE_{z_t \sim \PO^{\pi'}}[\TV(\bs(\cdot|z_t), \bstrue(\cdot|z_t))],
$$
where $z_t = \phi(\tau_t)$. Similarly, Eq.\eqref{eq:belief-guarantee} gives us
$$
\epsilon \ge \EE_{\tau_t \sim \PO^{\pi_b}}[\TV(\bs(\cdot|\tau_t), \bstrue(\cdot|\tau_t))] = \EE_{z_t \sim \PO^{\pi_b}}[\TV(\bs(\cdot|z_t), \bstrue(\cdot|z_t))].
$$
Performing change of measure w.r.t.~$z_t$ immediately completes the proof.
\end{proof}

\section{Discussion of Theorem~\ref{thm:2-stage}} \label{app:violate}
As mentioned in Footnote~\ref{ft:violate}, to directly apply our analysis for \os in the second stage we require realizability, i.e., the $M_{\POreal} \in \{ M_{\PO_k, \bs_{\PO_k}} \}$, where $\bs_{\PO_k}$ is the belief state approximation selected for $\PO_k$. This does not always hold, because for $\PO_{k^\star} = \POreal$, the selected $\bs_{\PO_{k^\star}}$ may not be its true belief state, causing the non-realizability. 

There are two fixes to this issue, both still leading to the kind of guarantee in Eq.\eqref{eq:os-with-stage-1-lss}: in the first fix, we can extend the analysis in Theorem~\ref{thm:scheffe-TV} to handle misspecification. In the second fix, we can change the algorithm as follows: 
\begin{enumerate}
\item Run Stage 2 with all $\{(\PO, \bs): \PO \in \{\PO_k\}, \bs\in\Bcal\}$. This way, realizability is satisfied and we obtain the guarantee for \os. 
\item When running the conditional selection algorithm for both \lss (Stage 1) and \os (Stage 2), we do not take the argmin of the score but the version space, i.e., the set of candidate conditionals that is plausible to be the true conditional. In the proof of Theorem~\ref{thm:scheffe-TV}, this corresponds to the set of $P_i$ whose score is no greater than $\epsilon'/6$ (see the paragraph below Eq.\eqref{eq:concentration-model-conditional}). It is easy to see that all conditional distributions in the version space enjoy the guarantee of Theorem~\ref{thm:scheffe-TV}.
\item For each $\PO \in \{\PO_k\}$, we pair it with each plausible belief state, and gather such pairs across $\{\PO_k\}$. Then, we take the intersection between this set and the version space of Stage 2. Any $(\PO, \bs)$ pair in the  intersection must enjoy both the guarantee of \os and that of \lss, thus satisfying the conditions of Theorem~\ref{thm:2-stage}. 
\end{enumerate}
\end{document}